\documentclass[11pt]{article}

\usepackage[margin=4cm]{geometry}

\title{De Re and De Dicto Knowledge\\ in Egocentric Setting}
\author{Pavel Naumov\\
School of Electronics and Computer Science\\
University of Southampton\\ United Kingdom\\
p.naumov@soton.ac.uk
\vspace{2mm}
\and 
Anna Ovchinnikova\\
School of Philosophy and Cultural Studies \\ Higher School of Economics\\ Russia\\
oaa232001@mail.ru
}

\date{}

\usepackage{amsmath}
\usepackage{wasysym}
\usepackage{amssymb}
\usepackage{graphics}

\newtheorem{theorem}{Theorem}
\newtheorem{lemma}{Lemma}

\newtheorem{definition}{Definition}

\renewcommand{\phi}{\varphi}
\renewcommand{\epsilon}{\varepsilon}

\newenvironment{proof}{\noindent{\em Proof.}}{\hfill $\boxtimes\hspace{2mm}$\linebreak}

\newenvironment{proof-of-claim}{\noindent{\em Proof of Claim.}}{\hfill $\boxtimes\hspace{2mm}$\linebreak}

\newcommand{\K}{{\sf K}}
\newcommand{\F}{{\sf F}}
\renewcommand{\L}{{\sf L}}

\newcommand{\R}{{\sf R}}
\newcommand{\D}{{\sf D}}
\newcommand{\W}{{\sf W}}
\usepackage{stmaryrd}
\renewcommand{\[}{\llbracket}
\renewcommand{\]}{\rrbracket}

\begin{document}
\maketitle

\begin{abstract}
Prior proposes the term ``egocentric'' for logical systems that study properties of agents rather than properties of possible worlds. In such a setting, the paper introduces two different modalities capturing de re and de dicto knowledge and proves that these two modalities are not definable through each other.
\end{abstract}


\section{Introduction}

Traditionally, the satisfaction relation in modal logic is defined as a relation $w\Vdash\phi$ between a possible world $w$ and a formula $\phi$. In such a setting, formula $\phi$ expresses a property of possible worlds. For example, statement $w\Vdash \text{``There are black holes''}$ expresses the fact that world $w$ has a property of containing black holes. It is also possible to consider logical systems that capture {\em properties of agents} rather than of possible worlds. In such systems, satisfaction relation $a\Vdash\phi$ is a relation between an agent $a$ and a formula $\phi$. Prior coins the term ``egocentric'' for such logical systems~\cite{p68nous}. In such a system, for example, statement $a\Vdash\text{``is a logician''}$ denotes the fact that agent $a$ has the property of being a logician. In egocentric logics, Boolean connectives can be used in the usual way. For example, the statement
$$
a\Vdash\text{``is a philosopher''}\wedge \neg \text{``is a logician''}
$$
means that agent $a$ is a philosopher, but not a logician. 

Seligman, Liu, and Girard propose the modality ``each friend'' $\F$ for egocentric logics~\cite{slg11lia,slg13tark}. In their language, the statement $a\Vdash\F\,\text{``is a logician''}$ means that all friends of agent $a$ are logicians. Modality $\F$ can be nested. For example, the statement $a\Vdash\F\,\F\,\text{``is a logician''}$ denotes the fact that all friends of agent $a$'s friends are logicians.
Jiang and Naumov introduce ``likes'' modality $\L$~\cite{jn22ijcai-preferences}. Using this modality, one can express the fact that agent $a$ likes logicians: $a\Vdash\L\,\text{``is a logician''}$. One can even say that agent $a$ does not like those who do not like logicians: $a\Vdash\neg\L\neg\L\,\text{``is a logician''}$. 

Grove and Halpern suggest to consider a more general form of egocentric setting, where the satisfaction relation $w,a\Vdash\phi$ is a relation between a world $w$, an agent $a$, and a formula $\phi$~\cite{gh91kr,gh93jlc,g95ai}. In such a setting, formula $\phi$ expresses a property of agent $a$ in world $w$. For example, the statement $w,a\Vdash \text{``is a logician''}$ means that agent $a$ is a logician in world $w$. If agent $a$ has a property $\phi$ in all worlds that agent $a$ cannot distinguish from the current world $w$, then we say that in world $w$ agent $a$ knows $\phi$ about {\em herself} and we denote it by $w,a\Vdash\K\phi$. For example, $w,a\Vdash\K\,\text{``is a great logician''}$ means that in world $w$ agent $a$ knows that she herself is a great logician. Modalities $\K$ and $\L$ can also be defined in this more general setting. Then, the statement $w,a\Vdash\K\,\F\,\L\,\text{``is a logician''}$ would mean that in world $w$ agent $a$ knows that all her friends like logicians. Epstein and Naumov propose a complete logical system for ``know who'' modality $\W$~\cite{en21aaai}. In their language, the statement $w,a\Vdash\W\,\text{``is a logician''}$ means that, in world $w$, agent $a$ knows which agent is a logician and the statement $w,a\Vdash\W\,\W\,\text{``is a logician''}$ means that, in world $w$, agent $a$ knows which agent knows who is a logician.

\begin{figure}
\begin{center}
\scalebox{0.55}{\includegraphics{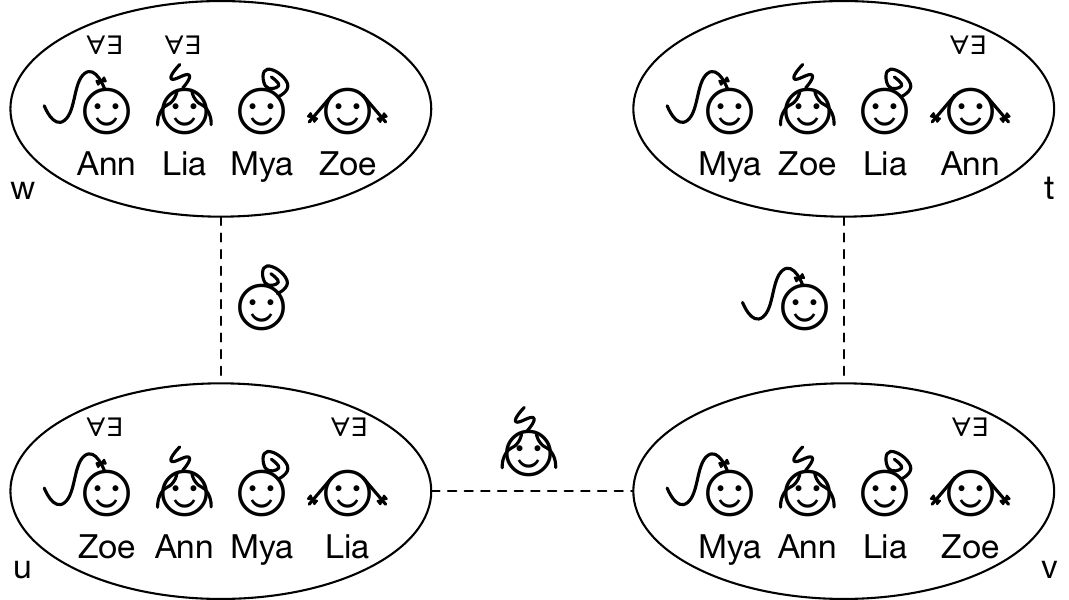}}
\caption{Epistemic model with extensions.}\label{intro figure}
\vspace{0mm}
\end{center}
\vspace{-2mm}
\end{figure}
In this paper, we consider agents with {\em nonrigid names} that might refer to different agents in different worlds~\cite{k80naming}. 
We assume that each name refers to exactly one agent in each possible world.
By {\em extension} $e_w(n)$ of a name $n$ in a world $w$ we mean the agent to which name $n$ refers in world~$w$. 
For example, in world $w$, depicted in Figure~\ref{intro figure}, name Lia refers to the second girl with fringe (bangs) on her face. Hence, $e_w(\text{Lia})$ is the described girl. In the figure, we use symbols $\forall\exists$ above an agent in a world to denote that the agent is a logician in that world. As one can see in the figure, agent $e_w(\text{Lia})$ is a logician in world $w$:
$$
w,e_w(Lia)\Vdash \text{``is a logician''}.
$$
Let us now turn our attention to world $u$. In this world, name Lia refers to a different girl than in world~$w$. However, as surprising as it may sound, Lia  from world $u$ (a girl with two braids) is also a logician:
$$
u,e_u(Lia)\Vdash \text{``is a logician''}.
$$
Name Mya refers to the same girl in both worlds: $e_w(Mya)=e_u(Mya)$. As shown in Figure~\ref{intro figure}, this girl cannot distinguish worlds $w$ and $u$. In spite of this, because Lia is a logician in both of the indistinguishable worlds, we can say in each of these worlds the agent $e_w(Mya)$ (also known as $e_u(Mya)$) knows that Lia is a logician. We write this as:
\begin{align}
w,e_w(Mya)&\Vdash \K_{Lia}\,\text{``is a logician''},\label{12-feb-a}\\
u,e_u(Mya)&\Vdash \K_{Lia}\,\text{``is a logician''}.\label{12-feb-b}
\end{align}

A careful reader might notice that Lia is not the only logician in worlds $w$ and~$u$:
\begin{align*}
w,e_w(Ann)&\Vdash \text{``is a logician''},\\ 
u,e_u(Zoe)&\Vdash \text{``is a logician''}.
\end{align*}
In other words, the left-most girl with a long single braid is a logician in both worlds that the agent $e_w(Mya)$ cannot distinguish. Thus, in both of these worlds, the agent $e_w(Mya)$ knows that the girl with a long braid is a logician:
\begin{align}
w,e_w(Mya)&\Vdash \K_{Ann}\,\text{``is a logician''},\label{12-feb-c}\\
u,e_u(Mya)&\Vdash \K_{Zoe}\,\text{``is a logician''}.\label{12-feb-d}
\end{align}
Although statements (\ref{12-feb-c}), (\ref{12-feb-d}), (\ref{12-feb-a}), and (\ref{12-feb-b}) capture the knowledge of agent $e_w(Mya)$ (also known as $e_u(Mya)$), the knowledge captured in statements (\ref{12-feb-c}) and (\ref{12-feb-d}) is {\em distinct} from the knowledge captured in statements (\ref{12-feb-a}) and (\ref{12-feb-b}). In the former case, the knowledge is about a specific agent (the long-braid girl) with different names. In the latter case, the knowledge is about the name ``Lia'' which refers to different agents. This distinction is known as de re/de dicto distinction~\cite{q56jp,l79pr,c76ps,ks19ohr,ws18aiml,ctw21tark}. 

To distinguish these two forms of knowledge, in this paper we represent them using two different modalities. We write $w,a\Vdash \R_n\phi$ if in world $w$ agent $a$ knows about the agent $e_w(n)$ that this agent has property $\phi$ in world $w$. We use the qualifier {\em de re} (about the thing) to refer to such knowledge. In our example, we will write statements (\ref{12-feb-c}) and (\ref{12-feb-d}) as 
\begin{align}
w,e_w(Mya)&\Vdash \R_{Ann}\,\text{``is a logician''},\nonumber\\
u,e_u(Mya)&\Vdash \R_{Zoe}\,\text{``is a logician''}.\label{12-feb-f}
\end{align}

At the same time, we write $w,a\Vdash \D_n\phi$ if in world $w$ agent $a$ knows about the name $n$ that the agent referred to by this name has property $\phi$. We use the qualifier {\em de dicto} (about what is said) to refer to such knowledge. In our example, we will write statements (\ref{12-feb-a}) and (\ref{12-feb-b}) as 
\begin{align}
w,e_w(Mya)&\Vdash \D_{Lia}\,\text{``is a logician''},\nonumber\\
u,e_u(Mya)&\Vdash \D_{Lia}\,\text{``is a logician''}.\label{22-feb-a}
\end{align}

Let us now turn our attention to worlds $t$ and $v$. Note that the last girl (with two braids) is a logician in both of these worlds. Thus, although the agent $e_t(Mya)$ (also known as $e_v(Mya)$) cannot distinguish these two worlds, this agent {\em de re} knows that the girl with two braids is a logician in both of these worlds:
\begin{align}
t,e_t(Mya)&\Vdash \R_{Ann}\,\text{``is a logician''},\nonumber\\
v,e_v(Mya)&\Vdash \R_{Zoe}\,\text{``is a logician''}.\label{12-feb-h}
\end{align}

Finally, let us consider worlds $u$ and $v$, indistinguishable by the girl with fringe on her face (also known as the agent $e_u(Ann)=e_v(Ann)$). Although this girl cannot distinguish the worlds, in both of these worlds, see statements~\eqref{12-feb-f} and \eqref{12-feb-h}, the agent with name Mya knows {\em de re} that Zoe is a logician. Hence, in both of these worlds, the girl with fringe {\em de dicto} knows that Mya knows {\em de re} that Zoe is a logician:
\begin{align*}
u,e_u(Ann)&\Vdash \D_{Mya}\R_{Zoe}\,\text{``is a logician''},\\
v,e_v(Ann)&\Vdash \D_{Mya}\R_{Zoe}\,\text{``is a logician''}.
\end{align*}

The contribution of this paper is three-fold. First, we propose formal semantics for {\em de re} and {\em de dicto} knowledge modalities $\R_n$ and $\D_n$. Second, we study the interplay between these two modalities. There are at least two different ways to study the interplay between logical connectives: to study the definability of these modalities through each other and to give a complete axiomatisation of all universal properties in the language containing both modalities. In this paper, we focus on the former approach and leave the latter for the future. We prove that modalities $\R_n$ and $\D_n$ are not definable through each other. Third, we consider an additional {\em agent-change} modality $@_n$. We show that formula $\D_n\phi$ is equivalent to formula $\R_n@_n\phi$ and, thus, {\em de dicto} knowledge modality $\D_n$ can be defined through modalities $\R_n$ and $@_n$. 
The converse, however, is not true: we prove that {\em de re} knowledge modality $\R_n$ cannot be defined through any combination of modalities $\D_n$ and $@_n$. Finally, we discuss an extension of our results to an egocentric setting with agent-specific names.

\section{Epistemic Model with Extensions}

In this section, we introduce the class of models that we use later to give a formal semantics of de re and de dicto knowledge modalities. Throughout the rest of the paper, we assume a fixed set of propositions and a fixed set $N$ of names.

\begin{definition}\label{model}
A tuple $\left(W,\mathcal{A},\{\sim_a\}_{a\in\mathcal{A}},\{e_w\}_{w\in W},\pi\right)$ is an epistemic model with extensions, where 
\begin{enumerate}
    \item $W$ is a (possibly empty) set of worlds,
    \item $\mathcal{A}$ is a (possibly empty) set of agents,
    \item $\sim_a$ is an ``indistinguishibility'' equivalence relation on the set of worlds $W$ for each agent $a\in\mathcal{A}$,
    \item $e_w$ is an ``extension'' function for each world $w\in W$ such that $e_w(n)\in \mathcal{A}$ for each name $n\in N$,
    \item $\pi(p)\subseteq W\times \mathcal{A}$ for each propositional variable $p$.
\end{enumerate}
\end{definition}

In our introductory example, the set $N$ of names is $\{Ann,Lia,Mya,Zoe\}$. The epistemic model with extensions for this example is depicted in Figure~\ref{intro figure}.  In the model, set $W$ contains four worlds: $w$, $u$, $v$, and $t$. Set $\mathcal{A}$ consists of the four girls, each of whom has a unique hairstyle. As usual, the indistinguishibility relation is shown in Figure~\ref{intro figure} using dashed lines. Note that this relation is labelled by agents, not names. This is because it is an agent who might or might not be able to distinguish two worlds. We allow a possibility that the same agent has different names in the worlds that she cannot distinguish. This would capture a situation when the agent does not know her own name. For example, an adopted child might not know the name she was given at birth. 

As discussed in the introduction, in each world $w$, the extension function $e_w$ assigns a name to each agent. Note that in the philosophy of language, the word ``extension'' is usually used to represent a set of objects denoted by a given name\footnote{``\dots we call
the set of things to which a common noun applies the extension of that common noun''~\cite[p.23]{l70synthese}.}. For example, the extension of the word ``game'' is the set containing chess,  battle of sexes, football, etc. In this paper, we assume that each name denotes a single agent. Thus, we assume that the value of the extension function is not a set, but a single agent. Our definitions can be generalised to the situation when, in a given world, the same name might denote multiple agents. In this case, each of the modalities $\R$ and $\D$ will need to be further split into ``knows about all'' and ``knows about some''.

Note that, in Definition~\ref{model}, the extension of a name depends on the world. One can also consider a more general setting when the extension also depends on the agent. This would allow modelling such names as ``my mother''. We further discuss this more general setting in Section~\ref{mother}.

Finally, by assuming that the value of $\pi(p)$ is a subset of $W\times\mathcal{A}$, we allow interpreting propositional variables as statements about the world and the agent. For example, if $p$ denotes the statement ``is a logician'', then $\pi(p)$ is the set of all pairs $(w,a)$ such that agent $a$ is a logician in world $w$. 

\section{Auxiliary Modality $@_n$}

In this section, we discuss an auxiliary modality $@_n$ that later we use to connect de re and de dicto knowledge modalities. Modality $@_n\phi$ means that statement $\phi$ is true about the agent with name $n$ in the current world. More formally, $w,a\Vdash@_n\phi$ means that $w,e_w(n)\Vdash\phi$. Note, in particular, that the validity of the statement $w,a\Vdash@_n\phi$ does not depend on agent $a$. To illustrate this modality, let us first observe that, in the introductory example depicted in Figure~\ref{intro figure}, Lia is a logician in world $w$:
$$w,e_w(Lia)\Vdash \,\text{``is a logician''}.$$
Thus, for {\em any} agent $a$,
$$w,a\Vdash @_{Lia}\,\text{``is a logician''}.$$
In particular, for $a=e_u(Zoe)$,
\begin{equation}\label{24-feb-a}
   w,e_u(Zoe)\Vdash @_{Lia}\,\text{``is a logician''}. 
\end{equation}
Note that Lia is also a logician in world $u$:
$$u,e_u(Lia)\Vdash \,\text{``is a logician''}.$$
Hence, similarly to \eqref{24-feb-a},
\begin{equation}\label{24-feb-b}
u,e_u(Zoe)\Vdash @_{Lia}\,\text{``is a logician''}.  
\end{equation}
Recall, see Figure~\ref{intro figure}, that the agent $e_u(Mya)$ (also known as $e_w(Mya)$) cannot distinguish worlds $w$ and $u$. Also, as we have seen in statements~\eqref{24-feb-a} and \eqref{24-feb-b}, formula $@_{Lia}\,\text{``is a logician''}$ is true about the agent $e_u(Zoe)$ in both of these worlds. Thus, in world $u$, Mya {\em de re} knows $@_{Lia}\,\text{``is a logician''}$ about  $e_u(Zoe)$: 
\begin{equation}\label{24-feb-c}
    u,e_u(Mya)\Vdash \R_{Zoe}@_{Lia}\,\text{``is a logician''}.
\end{equation}

As mentioned in the introduction, in this paper, we will show that formula  $\D_n\phi$ is logically equivalent to formula $\R_m@_n\phi$ for any formula $\phi$ and any names $m$ and $n$. The formal proof of this is given in Theorem~\ref{definability theorem}. The statement~\eqref{22-feb-a} and statement~\eqref{24-feb-c} illustrate this equivalence on a specific example.

\section{Syntax and Semantics}

In this paper, we consider the language $\Phi$ containing all three modalities that we have discussed above: $\R_n$, $\D_n$, and $@_n$. Formally, this language is defined by the following grammar:
$$
\phi:= p\;|\;\neg \phi\;|\;\phi\vee\phi\;|\;@_n\phi\;|\;\R_n\phi\;|\;\D_n\phi,
$$
where $p$ is a propositional variable and $n$ is a name. 
We read $@_n\phi$ as ``$\phi$ is true about $n$'', $\R_n\phi$ as ``de re knows $\phi$ about $n$'' and $\D_n\phi$ as ``de dicto knows $\phi$ about $n$''. We assume that other Boolean connectives are defined through negation and disjunction in the usual way.

\begin{definition}\label{sat}
For any world $w\in W$, any agent $a
\in \mathcal{A}$ of any epistemic model
with extensions $\left(W,\mathcal{A},\{\sim_a\}_{a\in\mathcal{A}},\{e_w\}_{w\in W},\pi\right)$ and any formula $\phi\in\Phi$, the satisfaction relation $w,a\Vdash\phi$ is defined as follows: 
\begin{enumerate}
    \item $w,a\Vdash p$ if $(w,a)\in \pi(p)$,
    \item $w,a\Vdash \neg\phi$ if $w,a\nVdash \phi$,
    \item $w,a\Vdash \phi\vee\psi$ if $w,a\Vdash \phi$ or $w,a\Vdash \psi$,
    \item $w,a\Vdash @_n\phi$ if $w,e_w(n)\Vdash\phi$,
    \item $w,a\Vdash \R_n\phi$ if $u,e_w(n)\Vdash\phi$ for each world $u\in W$ such that $w\sim_{a}u$,
    \item $w,a\Vdash \D_n\phi$ if $u,e_u(n)\Vdash\phi$ for each world $u\in W$ such that $w\sim_{a}u$.
\end{enumerate}
\end{definition}

Note that item~5 requires formula $\phi$ to be true {\em about the same} agent $e_w(n)$ in all indistinguishible worlds. This means that it captures {\em de re} knowledge about the agent $e_w(n)$. At the same time, item~6 requires $\phi$ to be true about the agent that has the name $n$ in each indistinguishible world $u$. Thus, it captures {\em a property of name} $n$ rather than {\em a property of agent} $e_w(n)$. In other words, it expresses {\em de dicto} knowledge.

\begin{definition}\label{truth set definition}
If $\left(W,\mathcal{A},\{\sim_a\}_{a\in\mathcal{A}},\{e_w\}_{w\in W},\pi\right)$ is any fixed epistemic model with extensions, then the truth set $\[\phi\]$ of a formula $\phi\in\Phi$ is the set $\{(w,a)\in W\times\mathcal{A}\;|\;w,a\Vdash \phi\}$.
\end{definition}

\begin{definition}
Formulae $\phi,\psi\in\Phi$ are semantically equivalent if $\[\phi\]=\[\psi\]$ for each epistemic model with extensions.
\end{definition}
In other words, formulae $\phi,\psi\in\Phi$ are semantically equivalent if $w,a\Vdash\phi$ iff $w,a\Vdash\psi$ for each world $w$ and each agent $a$ of each epistemic model with extensions.

\section{De Re/De Dicto Distinction in Logic Literaure}

Term {\em de re} and {\em de dicto} are used in the literature in a non-consistent way. In the words of the Stanford Encyclopedia of Philosophy, ``The {\em de re/de dicto} distinction has meant different things to different people''~\cite{n23plato}. Some authors use terms ``de re reading'' and ``de dicto reading'' to distinguish $\exists\forall$ and $\forall\exists$ order of quantifiers, especially if the universal quantifier is represented by a modality~\cite{gj19lori,ay13jlc,f20kpd}. It is interesting to point out that the journal version \cite{gj21jlli} of \cite{gj19lori} switches from terms de re/de dicto to proactive/reactive. This use of terms de re/de dicto is not closely related to the way we use terms de re/de dicto knowledge in this paper. Indeed we use terms {\em de re} and  {\em de dicto} to distinguish the knowledge about the object and the name. There are no names in the works we have mentioned. In addition, the distinction between items 5 and 6 of Definition~\ref{sat} does not appear to be reducible to a change in the order of quantifiers. 

At the same time, Wang, Wei, and Seligman propose a closely related to our logical system capable of capturing de re and de dicto forms of knowledge~\cite{wws22apal}. The main difference between their approach and ours is the semantics and the language. They use traditional, non-egocentric, semantics that defines satisfaction as a relation between a possible world and a formula. To capture de re and de dicto knowledge, they extend the language with {\em agent variables} and a special ``update''~\cite{ctw21tark} modality $[x:=n]$ that changes the value of variable $x$ to the value of the extension function on name $n$ in the current world. Using their notations, one can write our de re knowledge statement~(\ref{12-feb-f}) as:
$$
u\Vdash [x:=\textit{``Mya''}]\,\K_{Zoe}\,\text{``is a logician$(x)$''}.
$$
And our de dicto knowledge statement~(\ref{22-feb-a}) as
$$
u\Vdash \K_{Lia}\,\text{``is a logician$(Mya)$''}.
$$
By using egocentric semantics, we avoid the need to add variables to our language. Our approach also allows us to define de re and de dicto knowledge as primitive modalities. Having primitive modalities in the language creates an opportunity to study the definability/undefinability of these two notions through each other. On the other hand, their approach is not limited to de re/de dicto distinction and potentially could be used to express other concepts about non-rigid names.

\section{Technical Results}

In this section, we fully analyse which of modalities $\R$, $\D$, and $@$ are definable through which of the others. We start with the definability result we mentioned earlier.

\subsection{Definability of $\D$ through $\R$ and $@$}

\begin{theorem}\label{definability theorem}
$w,a\Vdash \D_n\phi$ iff $w,a\Vdash \R_m@_n\phi$ for any formula $\phi\in\Phi$, any names $n,m\in N$ as well as any world $w\in W$ and any agent $a\in \mathcal{A}$ of any epistemic model
with extensions $\left(W,\mathcal{A},\{\sim_a\}_{a\in\mathcal{A}},\{e_w\}_{w\in W},\pi\right)$.
\end{theorem}
\begin{proof}
$(\Rightarrow):$ Consider any world $u$ such that $w\sim_a u$. By item~5 of Definition~\ref{sat}, it suffices to show that $u,e_w(m)\Vdash @_n\phi$. Then, by item~4 of the same Definition~\ref{sat}, it suffices to prove that $u,e_u(n)\Vdash \phi$. The last statement follows from the assumption $w,a\Vdash \D_n\phi$ by item~6 of Definition~\ref{sat}.

\vspace{1mm}\noindent
$(\Leftarrow):$ Consider any world $u$ such that $w\sim_a u$. 
Note that by item~6 of Definition~\ref{sat}, it suffices to show that
$u,e_u(n)\Vdash @_n\phi$. Indeed, the assumption $w,a\Vdash \R_m@_n\phi$ implies $u,e_w(m)\Vdash @_n\phi$ by item~5 of Definition~\ref{sat} and the assumption $w\sim_a u$. Therefore, $u,e_u(n)\Vdash \phi$ by item~4 of Definition~\ref{sat}.
\end{proof}

\subsection{Undefinability of $\D$ through $\R$}

\begin{figure*}
\begin{center}
\scalebox{0.55}{\includegraphics{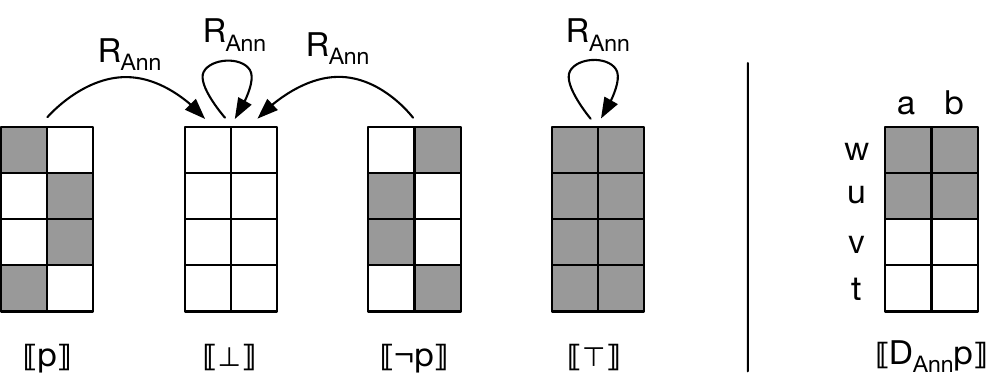}}
\caption{Towards the proof of Theorem~\ref{D undefinable theorem}. In the model, $w\sim_x u$, $v\sim_x t$ for each $x\in\{a,b\}$. Also, $e_w(Ann)=e_v(Ann)=a$ and $e_u(Ann)=e_t(Ann)=b$.}\label{DviaR figure}
\vspace{0mm}
\end{center}
\end{figure*}

As we have seen in Theorem~\ref{definability theorem}, modality $\D$ can be defined through a combination of modalities $\R$ and $@$. It is clear that modality $\D$ cannot be defined through modality $@$ alone because the semantics of modality $@$ does not refer to the indistinguishibility relation, see item~4 of Definition~\ref{sat}. In this subsection, we show that modality $\D$ also cannot be defined through just modality $\R$. The proof is using a recently proposed ``truth set algebra'' technique~\cite{kn22arxiv}. Unlike the more traditional ``bisimulation'' method, the new technique uses only one model instead of two. We introduce the truth set algebra technique while proving our first undefinability result in this subsection. We use the same technique in the next subsection to prove the undefinability of modality $\R$ through $\D$ and $@$.

Without loss of generality, assume that language $\Phi$ contains a single propositional variable $p$ and a single name $Ann$. Consider an epistemic model with extensions that has four possible worlds, $w$, $u$, $v$, and $t$ and two agents, $a$ and $b$. We assume that each of the agents cannot distinguish world $w$ from world $u$ and world $v$ from world $t$. That is:  $w\sim_x u$, $v\sim_x t$ for each $x\in\{a,b\}$.
Let the name $Ann$ refer to agent $a$ in worlds $w$ and $v$ and the same name $Ann$ refer to agent $b$ in worlds $v$ and $t$. In other words, $e_w(Ann)=e_v(Ann)=a$, $e_u(Ann)=e_t(Ann)=b$.
Finally, let $\pi(p)=\{(w,a),(u,b),(v,b),(t,a)\}$.

We visualise the truth set $\[\phi\]$ of an arbitrary formula $\phi\in\Phi$ using $4\times 2$ diagrams whose rows correspond to the possible worlds and columns to the agents in our model. In such a diagram, we shade a cell grey if the corresponding world/agent pair belongs to the set $\[\phi\]$. The left-most diagram in Figure~\ref{DviaR figure} visualises the truth set $\[p\]$. For example, the upper-left cell of that diagram is grey because $(w,a)\in \pi(p)$ and, thus, $(w,a)\in \[p\]$.
The other three diagrams to the left of the vertical bar in Figure~\ref{DviaR figure} visualise the truth sets $\[\bot\]$, $\[\neg p\]$, and $\[\top\]$.

\begin{lemma}\label{induction step lemma DviaR}
For any formula $\phi\in \Phi$, if $\[\phi\]\in \{\[p\],\[\bot\],\[\neg p\], \[\top\]\}$, then $\[\R_{Ann}\phi\]\in\{\[p\],\[\bot\]\,\[\neg p\], \[\top\]\}$.
\end{lemma}
\begin{proof}
We consider the following four cases separately.

\noindent{\em Case I}: $\[\phi\]=\[p\]$. It suffices to show that $\[\R_{Ann}\phi\]=\[\bot\]$.
Consider any possible world $s\in \{w,u,v,t\}$ and any agent $x\in \{a,b\}$. To prove $\[\R_{Ann}\phi\]=\[\bot\]$, it is enough to show that $(s,x) \notin\[\R_{Ann} \phi\]$, see the second diagram from the left one in Figure~\ref{DviaR figure}.

Observe that the left-most diagram in Figure~\ref{DviaR figure} has the property that for any column $c\in \{a,b\}$ there are rows $r_1\in\{w,u\}$ and $r_2\in\{v,t\}$ such that cells $(r_1,c)$ and $(r_2,c)$ are white. Hence, because $\[\phi\]=\[p\]$, for any column $c\in \{a,b\}$ there are rows $r_1\in\{w,u\}$ and $r_2\in\{v,t\}$ such that $(r_1,c)\notin \[\phi\]$ and $(r_2,c)\notin \[\phi\]$. In particular, there are rows 
\begin{equation}\label{26-feb-a}
   r_1\in\{w,u\}\;\;\;\text{and}\;\;\;r_2\in\{v,t\} 
\end{equation}
such that $(r_1,e_s(Ann))\notin \[\phi\]$ and $(r_2,e_s(Ann))\notin \[\phi\]$.  Thus, by Definition~\ref{truth set definition},
\begin{equation}\label{26-feb-b}
   r_1,e_s(Ann)\nVdash \phi\;\;\;\text{and}\;\;\;r_2,e_s(Ann)\nVdash \phi. 
\end{equation}
Recall that $w\sim_x u$ and $v\sim_x t$ by the definition of the epistemic model with extensions that we consider. Thus, due to statement~\eqref{26-feb-a}, there is $i\in\{1,2\}$ such that $s\sim_x r_i$. In addition, $r_i,e_s(Ann)\nVdash \phi$ by statement~\eqref{26-feb-b}. Hence, $s,x\nVdash \R_{Ann}\phi$ by item~5 of Definition~\ref{sat}. Then, $(s,x) \notin\[\R_{Ann} \phi\]$ by Definition~\ref{truth set definition}.

Therefore, if $\[\phi\]=\[p\]$, then $\[\R_{Ann}\phi\]=\[\bot\]$. In Figure~\ref{DviaR figure}, we visualise this fact by an arrow from diagram $\[p\]$ to diagram $\[\bot\]$ labelled with $\R_{Ann}$.

\noindent{\em Case II}: $\[\phi\]=\[\neg p\]$. The proof, in this case, is similar to the previous case.

\noindent{\em Case III}:
Suppose that $\[\phi\]=\[\bot\]$. It suffices to show that $\[\R_{Ann}\phi\]=\[\bot\]$.
Consider any world $s\in \{w,u,v,t\}$ and any agent $x\in \{a,b\}$. To prove $\[\R_{Ann}\phi\]=\[\bot\]$, it is enough to establish that $(s,x) \notin\[\R_{Ann} \phi\]$. 

The assumption $\[\phi\]=\[\bot\]$ implies that $\[\phi\]=\varnothing$, see the second diagram from the left one in Figure~\ref{DviaR figure}. Hence, $(s,x)\notin\[\phi\]$. Then, $s,x\nVdash\phi$ by Definition~\ref{truth set definition}. Thus, $s,x\nVdash \R_{Ann}\phi$ by item~5 of Definition~\ref{sat}. Therefore, $(s,x) \notin\[\R_{Ann} \phi\]$ by Definition~\ref{truth set definition}.

\noindent{\em Case IV}:
Suppose that $\[\phi\]=\[\top\]$. It suffices to show that $\[\R_{Ann}\phi\]=\[\top\]$. The assumption $\[\phi\]=\[\top\]$ implies that $s,x \in \[\phi\]$ for any world $s\in \{w,u,v,t\}$ and any agent $x\in \{a,b\}$, see the fourth diagram in Figure~\ref{DviaR figure}. Hence, $s,x \Vdash \phi$ for any world $s\in \{w,u,v,t\}$ and any agent $x\in \{a,b\}$, by Definition~\ref{truth set definition}. Thus, $s,x \Vdash \R_{Ann}\phi$ for any world $s\in \{w,u,v,t\}$ and any agent $x\in \{a,b\}$, by item~5 of Definition~\ref{sat}. Therefore, $\[\R_{Ann}\phi\]=\[\top\]$ by Definition~\ref{truth set definition}.
\end{proof}

\begin{figure*}
\begin{center}
\scalebox{0.55}{\includegraphics{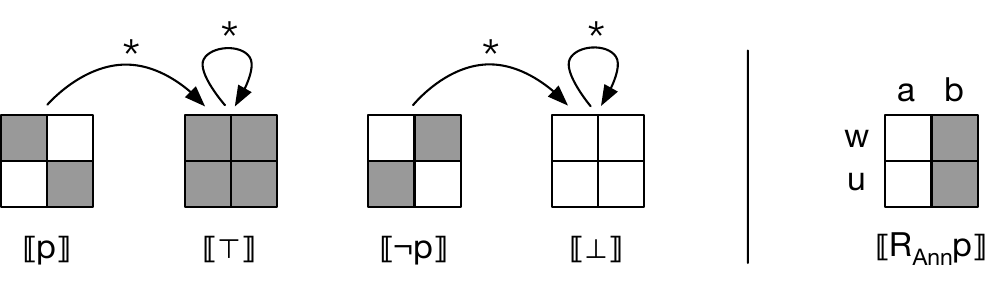}}
\caption{Towards the proof of Theorem~\ref{R undefinable theorem}. The symbol $\star$ on the arrows means that the arrow is labelled with both modalities,  $\D_{Ann}$ and $@_{Ann}$. In the model, $w\sim_a u$, $e_w(Ann)=a$, and $e_u(Ann)=b$.}\label{RviaD@ figure}
\vspace{0mm}
\end{center}
\end{figure*}

\begin{lemma}\label{induction lemma RviaD}
$\[\phi\]\in \{\[p\],\[\bot\],\[\neg p\],\[\top\]\}$ for any formula $\phi\in\Phi$ that does not use modalities $\D$ and $@$.
\end{lemma}
\begin{proof}
We prove the statement of the lemma by induction on the structural complexity of formula $\phi$. The statement of the lemma is true for propositional variable $p$ because the truth set $\[p\]$ is an element of the set $\{\[p\],\[\bot\],\[\neg p\],\[\top\]\}$.

Assume that formula $\phi$ has the form $\neg\psi$. Note, by item~2 of Definition~\ref{sat} and Definition~\ref{truth set definition}, the set $\[\neg\psi\]$ is the {\em complement} of the truth set $\[\psi\]$. Observe also, that the complement of each of the sets in the family $\{\[p\],\[\bot\],\[\neg p\],\[\top\]\}$  also belongs to this family. For example, the complement of the set $\[\neg p\]$ is the set $\[p\]$, see Figure~\ref{DviaR figure}. By the induction hypothesis, $\[\psi\]\in \{\[p\],\[\bot\],\[\neg p\],\[\top\]\}$. Therefore, $\[\phi\]=\[\neg\psi\]\in \{\[p\],\[\bot\],\[\neg p\],\[\top\]\}$.

Next, suppose that formula $\phi$ has the form $\psi_1\vee \psi_2$. Note, by Definition~\ref{sat} and Definition~\ref{truth set definition}, the set $\[\psi_1\vee \psi_2\]$ is the {\em union} of the sets $\[\psi_1\]$ and $\[\psi_2\]$. Observe that the union of any two sets in the family $\{\[p\],\[\bot\],\[\neg p\],\[\top\]\}$  also belongs to this family. For example, the union of the sets $\[p\]$ and $\[\neg p\]$ is the set $\[\top\]$, see Figure~\ref{DviaR figure}. By the induction hypothesis, $\[\psi_1\], \[\psi_2\]\in \{\[p\],\[\bot\],\[\neg p\],\[\top\]\}$. Therefore, $\[\phi\]=\[\neg\psi\]\in \{\[p\],\[\bot\],\[\neg p\],\[\top\]\}$.

Finally, if formula $\phi$ has the form $\R_\text{Ann}\psi$, then the statement of the lemma follows from the induction hypothesis by Lemma~\ref{induction step lemma DviaR}.
\end{proof}

\begin{lemma}\label{outside of closure DviaR}
$\[\D_{Ann}p\]\notin \{\[p\],\[\bot\],\[\neg p\],\[\top\]\}$.   
\end{lemma}
\begin{proof}
It suffices to show that the truth set $\[\D_{Ann}p\]$ is depicted in the right-most diagram of Figure~\ref{DviaR figure}. Indeed, 
$w,a\Vdash p$
and 
$u,b\Vdash p$ 
by the choice of function $\pi$ in our model and item~1 of Definition~\ref{sat}. 
Then, 
$
w,e_w(Ann)\Vdash p 
$
and
$u,e_u(Ann)\Vdash p$
by the definition of the extension function in our model.
Recall that, in our model, agents $a$ and $b$ cannot distinguish worlds $w$ and $u$, but they {\em can} distinguish these worlds from worlds $v$ and $t$. Hence, by item~6 of Definition~\ref{sat},
\begin{equation}\label{3-mar-a}
s,x\Vdash \D_{Ann}p \;\;\;\text{ for each $s\in\{w,u\}$ and each $x\in\{a,b\}$}.  
\end{equation}

At the same time, $v,a\nVdash p$ again by the choice of function $\pi$ in our model and item~1 of Definition~\ref{sat}. Then, $v,e_v(Ann)\nVdash p$ by the definition of the extension function in our model. Recall that agents $a$ and $b$ cannot distinguish worlds $v$ and $t$. Thus, by item~6 of Definition~\ref{sat},
\begin{equation}\label{3-mar-b}
s,x\nVdash \D_{Ann}p \;\;\;\text{ for each $s\in\{v,t\}$ and each $x\in\{a,b\}$}.  
\end{equation}
Statements~\eqref{3-mar-a} and \eqref{3-mar-b} imply that the set $\[\D_{Ann}p\]$ is depicted in the right-most diagram of Figure~\ref{DviaR figure}.
\end{proof}

The next theorem follows from the two lemmas above.
\begin{theorem}[undefinability]\label{D undefinable theorem}
Formula $\D_{Ann} p$ is not semantically equivalent to any formula in language $\Phi$ that does not use modalities $\D$ and $@$.
\end{theorem}

Note that in the above theorem and the rest of this paper we interpret undefinability as being able to express a property which is not expressible otherwise. Thus, we consider a ``non-uniform'' definability that {\em does not} require to express modality $\D$ through a fixed syntactical combination of other connectives (for example, as in Theorem~\ref{definability theorem}). Observe that our definition of {\em definability} is weaker. Thus, our results about {\em undefinability} are stronger and they imply similar results for uniform definability. The {\em definability} result in this paper (Theorem~\ref{definability theorem}) is given for uniform definability and, thus, is also in the strongest form.

\subsection{Undefinability of $\R$ via $\D$ and $@$}

In this subsection, we prove that de re knowledge modality $\R$ is not definable through de dicto knowledge modality $\D$ even if modality $@$ is also used. 
Without loss of generality, we again assume that language $\Phi$ contains a single propositional variable $p$ and a single name $Ann$.
The proof of undefinability uses the same truth set algebra methods as in the previous subsection, but the epistemic model is different.  The new model has only two worlds, $w$ and $u$, and two agents, $a$ and $b$. Agent $a$ cannot distinguish the two worlds, but agent $b$ can. We assume that the name Ann refers to agent $a$ in world $w$ and to agent $b$ in world $u$. That is, $e_w(Ann)=a$ and $e_u(Ann)=b$. Finally, let $\pi(p)=\{(w,a),(u,b)\}$.

The proof of the next lemma is similar to the proof of Lemma~\ref{induction step lemma DviaR}, but instead of Figure~\ref{DviaR figure} it uses Figure~\ref{RviaD@ figure}. In the latter figure, the symbol $\star$ on an arrow means that the arrow is labelled with both modalities,  $\D_{Ann}$ and $@_{Ann}$.

\begin{lemma}\label{induction step lemma RviaD@}
$\[\D_{Ann}\phi\],\[@_{Ann}\phi\]\in\{\[p\],\[\bot\]\,\[\neg p\], \[\top\]\}$
for any formula $\phi\in \Phi$ such that
$\[\phi\]\in \{\[p\],\[\bot\],\[\neg p\], \[\top\]\}$.
\end{lemma}

The proof of the next lemma is similar to the proof of Lemma~\ref{induction lemma RviaD}, but instead of Lemma~\ref{induction step lemma DviaR} it uses Lemma~\ref{induction step lemma RviaD@}.
\begin{lemma}\label{induction lemma DviaR@}
$\[\phi\]\in \{\[p\],\[\bot\],\[\neg p\],\[\top\]\}$ for any formula $\phi\in\Phi$ that does not use modality $\R$.
\end{lemma}

\begin{lemma}\label{outside of closure RviaD@}
$\[\R_{Ann}p\]\notin \{\[p\],\[\bot\],\[\neg p\],\[\top\]\}$.   
\end{lemma}
\begin{proof}
It suffices to show that the truth set $\[\R_{Ann}p\]$ is depicted in the right-most diagram of Figure~\ref{RviaD@ figure}. Indeed, $u,a\nVdash p$ by choice of the function $\pi$ in our model and item~1 of Definition~\ref{sat}. Then, $u,e_w(Ann)\nVdash p$ by the definition of the extension function in our model. Hence, by item~5 of Definition~\ref{sat},
\begin{equation}\label{3-mar-c}
s,a\nVdash \R_{Ann} p \;\;\;\;\text{ for each world $s\in \{w,u\}$.}   
\end{equation}
because, in our model, $w\sim_a u$.
At the same time, 
$
w,a\Vdash p
$
and
$
u,b\Vdash p$
by the choice of function $\pi$ in our model and item~1 of Definition~\ref{sat}.
Hence, 
$
w,e_w(Ann)\Vdash p 
$
and 
$u,e_u(Ann)\Vdash p$
by the definition of the extension function in our model.
Recall that agent $b$ can distinguish worlds $w$ and $u$. Thus, by item~5 of Definition~\ref{sat},
\begin{equation}\label{3-mar-d}
w,b\Vdash \R_{Ann}p  \;\;\;\text{ and }\;\;\;   u,b\Vdash \R_{Ann}p.
\end{equation}
Statements~\eqref{3-mar-c} and \eqref{3-mar-d} imply that the set $\[\R_{Ann}p\]$ is depicted in the right-most diagram of Figure~\ref{RviaD@ figure}.
\end{proof}

The next theorem follows from the two lemmas above.
\begin{theorem}[undefinability]\label{R undefinable theorem}
Formula $\R_{Ann} p$ is not semantically equivalent to any formula in language $\Phi$ that does not use modalities $\R$.
\end{theorem}

\section{Agent-specific names}\label{mother}

\begin{figure*}
\begin{center}
\scalebox{0.55}{\includegraphics{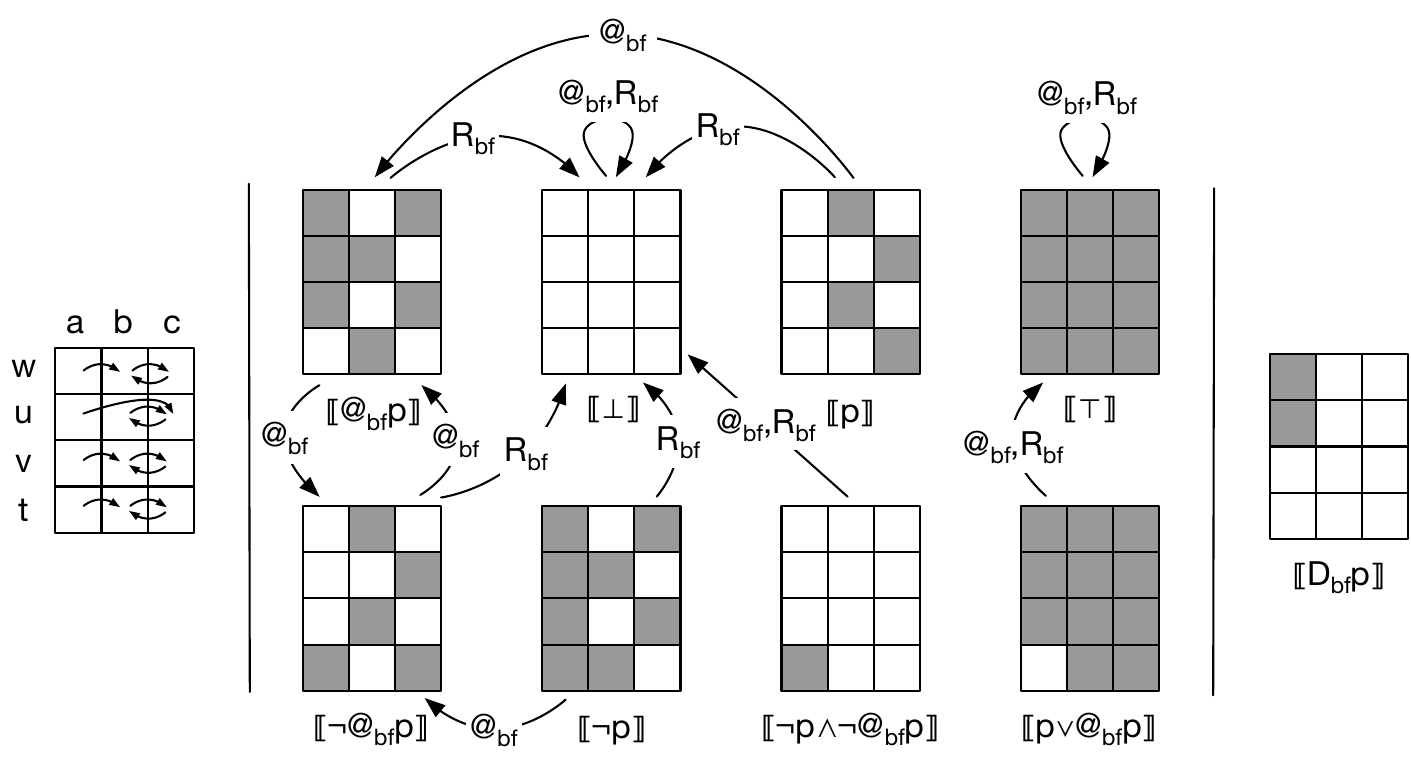}}
\caption{Towards the proof of the fact that modality $\D$ is not definble through modalities $\R$ and $@$ in the setting with agent-specific names.}\label{DviaR@ mother figure}
\vspace{0mm}
\end{center}
\end{figure*}

Throughout the paper, we have assumed that names are world-specific: in different worlds, the same name can refer to different agents. At the same time, our names have {\em not} been agent-specific. Agent-specific names, like ``ma'' (for mother), are names whose meaning depends not only on the world but also on the agent. Agent-specific names can be added to our system by allowing, in Definition~\ref{model}, the extension functions $e^a_w(n)$ that depend not only on world $w\in W$ but also on agent $a\in\mathcal{A}$. Definition~\ref{sat} could be adjusted accordingly. For example, its item~4 would now look like this:
$w,a\Vdash @_n\phi$ iff $w,e^a_w(n)\Vdash\phi$.

Then, the statement $w,a\Vdash @_{ma}\,\text{``is a logician''}$ means that the mother of agent $a$ is a logician in world $w$. At the same time, the statement $w,a\Vdash  @_{ma}@_{ma}\,\text{``is a logician''}$ means that {\em the grandma} (on the mother's side) of agent $a$ is a logician in world $w$.

Because our original setting with non-agent-specific names is a special case of the setting with agent-specific names, undefinability results from Theorem~\ref{D undefinable theorem} and Theorem~\ref{R undefinable theorem} and their proofs remain valid for agent-specific names. However, {even the statement} of Theorem~\ref{definability theorem} is no longer true. In fact, {\em surprisingly}, in the case of the setting with agent-specific names, modality $\D$ is undefinable through modalities $\R$ and $@$. 

To prove this, we assume, without the loss of generality, that our language consists of a single name $bf$ (short for ``best friend''). We consider an epistemic model with extensions that has four worlds: $w$, $u$, $v$, $t$ and three agents: $a$, $b$, $c$. We assume that all three agents cannot distinguish world $w$ from world $u$ and they all also cannot distinguish world $v$ from world $t$. The arrows in the left-most diagram in Figure~\ref{DviaR@ mother figure} specify the extension of the name ``bf'' in different worlds for different agents. For example, the arrow from cell $(u,a)$ to cell $(u,c)$ denotes that $e_u^a(bf)=c$. In other words, agent $a$ in world $u$ uses the name ``bf'' to refer to agent $c$.

By $\mathcal{F}$ we denote the family of truth sets  $\{\[@_{bf}p\],\[\bot\],\[p\],\[\top\]$, $\[\neg@_{bf}p\],\[\neg p\]$, $\[\neg p\wedge \neg@_{bf}p\],\[p\vee @_{bf}p\]\}$. 
The proof of the next lemma is similar to the proof of Lemma~\ref{induction step lemma DviaR}, but instead of Figure~\ref{DviaR figure} it uses Figure~\ref{DviaR@ mother figure}. 

\begin{lemma}\label{induction step lemma mother}
$\[@_{bf}\phi\],\[\R_{bf}\phi\]\in\mathcal{F}$
for any $\phi\in \Phi$ such that
$\[\phi\]\in \mathcal{F}$.
\end{lemma}

The proof of the next lemma is similar to the proof of Lemma~\ref{induction lemma RviaD}, but instead of Lemma~\ref{induction step lemma DviaR} it uses Lemma~\ref{induction step lemma mother}.
\begin{lemma}\label{induction lemma mama}
$\[\phi\]\in \mathcal{F}$ for any $\phi\in\Phi$ that does not use modality $\D$.
\end{lemma}

\begin{lemma}\label{outside of closure mother}
$\[\D_{bf}p\]\notin \mathcal{F}$.   
\end{lemma}
\begin{proof}
We show the diagram for the set  $\[\D_{bf}p\]$ at the right of Figure~\ref{DviaR@ mother figure}.   
\end{proof}

The next theorem follows from the two lemmas above.
\begin{theorem}[undefinability]\label{D undefinable theorem mother}
Formula $\D_{bf} p$ is not semantically equivalent in a setting with agent-specific names to any formula in language $\Phi$ that does not use modalities $\D$.
\end{theorem}
 
Note that if language $\Phi$ is extended by {\em name constant} ``se'' (self) and it is assumed that $e^a_w(se)=a$ for each agent $a$ in each world $w$, then modality $\D$ becomes again definable:
$
\D_n\phi \equiv \R_{se}@_{n}\phi
$.

\section{Conclusion}

In this paper, we investigated the formal definitions of de re and de dicto knowledge in an egocentric setting. Our main result is the undefinability of these notions through each other. In the future, we plan to develop a sound and complete logical system capturing the interplay between these two notions.

\bibliographystyle{plain}
\bibliography{naumov}
\end{document}